\documentclass[11pt,a4paper]{article}

\usepackage[utf8]{inputenc}
\usepackage[T1]{fontenc}

\usepackage{amssymb,amsmath,amsfonts,amsxtra,amsthm}
\usepackage{amsfonts}
\usepackage{graphicx}
\usepackage{alltt}

\usepackage[lined,boxed]{algorithm2e}
\usepackage[usenames,dvipsnames]{pstricks}

\usepackage{color}
\usepackage{pifont}

\usepackage{geometry}

\newtheorem{definition}{Definition}

\newtheorem{example}{Example}

\newtheorem{lemma}{Lemma}
\newtheorem{theorem}{Theorem}

\begin{document}

\sloppy

\title{On the expressive power of unit resolution}
\author{Olivier Bailleux}

\maketitle

\begin{abstract}
This preliminary report addresses the expressive power of unit resolution regarding input data encoded with partial truth assignments of propositional variables. A characterization of the functions that are computable in this way, which we propose to call propagatable functions, is given. By establishing that propagatable functions can also be computed using monotone circuits, we show that there exist polynomial time complexity propagable functions  requiring an exponential amount of clauses to be computed using unit resolution. These results shed new light on studying \textsc{cnf} encodings of \textsc{np}-complete problems in order to solve them using propositional satisfiability algorithms. A paper is being drafted, which aims to present the concepts introduced in the present report and the underlying scientific issues in a more simple way. 
\end{abstract}



\section{Introduction and motivations}

Unit resolution is a key feature of state of the art \textsc{sat} solvers \cite{Moskewicz2001} \cite{Goldberg2002} \cite{minisat}, where it speeds up the search for solutions and inconsistencies.

It is well known that different \textsc{cnf} representations of a given problem do not always allow unit resolution to deduce the same information. For example, the \textsc{cnf} encoding for pseudo Boolean constraints proposed in \cite{jsat2006-pseudotocnf} allows unit resolution to restore generalized arc consistency. This is not the case with the encoding proposed in \cite{Warners-98}, which does not allow unit resolution to deduce as much information as the former encoding does. As a manner of speaking, the expressive power of unit resolution is best exploited using the encoding proposed in \cite{jsat2006-pseudotocnf}, with notable consequences on the resolution time.

Two important related questions are "What information can be deduced by unit resolution?" and "Which clauses are required in order to allow this information to be deduced?"

These questions are strongly connected to the characterization of the application field of \textsc{sat} solvers: "Which problems can be solved as efficiently using a \textsc{sat} solver as using a specialized solver?" and "How to encode these problems into \textsc{cnf} formulae for optimal resolution time?"

In this paper, we are interested in the functions that can be calculated by means of unit resolution. Studying the expressive power of unit resolution requires characterizing these functions, which will be called \textit{propagatable functions}, and specifying the size of the formulae required to compute them.

Section \ref{relatedworks} presents the three main research directions related to the expressive power of unit resolution.
Section \ref{propagatablefunctions} introduces the concept of  propagators and propagatable functions as a formal framework where unit resolution is a computing model. This section also presents theoretical results that will be used in section \ref{expressivepower}, where the expressive power of unit resolution is compared to the one of monotone Boolean circuits.
Section \ref{synthesis} ends the paper with a synthesis of the results, which highlight their implications regarding the efficiency of unit resolution as a filtering technique in SAT solvers.

\section{Related works \label{relatedworks}}

There are at least three research directions related to the study of the expressive power of unit resolution.

The first one aims to identify the classes of formulae for which unit resolution is a complete refutation procedure in the sense that it produces the empty clause if and only if the input formula is not satisfiable. For example, this property holds for the formulae containing only Horn clauses
\cite{henschen-74}. This approach differs from that proposed in this paper since it considers the formulae as input data instead of computing systems.

The second direction aims to characterize the complexity of determining whether a given formula can be refuted by unit resolution or not.
This decision problem denoted \textsc{unit} is known to be \textsc{p}-complete, meaning that for any problem $\pi$ with polynomial time complexity, there exists a log space reduction from $\pi$ to \textsc{unit} \cite{Jones1976}. Circuit value and monotone circuit value, which consist to determine the output value of a Boolean circuit (monotone Boolean circuit, respectively), given its input values, are both \textsc{p}-complete too \cite{goldschlager77}. Regarding the complexity theory, \textsc{unit}, circuit value and monotone circuit value have then the same expressive power.
In the present paper, a different point of view is adopted. The \textsc{cnf} formula is not the input data of a program, but the program itself. The input data is a partial truth assignment encoded in a natural way, i.e., each input variable can be either assigned to \texttt{true}, assigned to \texttt{false}, or not assigned. Similarly, the Boolean circuits are not considered as inputs of a program, but as a programs by themselves. In this context, circuit value and monotone circuit value have not the same expressive power for at least two reasons : (1) monotone circuits can only compute monotone Boolean functions, while any Boolean function can be computed using a general Boolean circuit, and (2) there exist monotone Boolean functions which can be computed by polynomially sized Boolean circuits, but requiring an exponential number of gates to be computed using monotone circuits \cite{monotone-complexity-matching}.
One of our results is that used as a computation model, unit resolution  with natural input encoding has the same expressive power as monotone Boolean circuits, then less expressive power than general Boolean circuits.
Obviously, the input encoding plays a central role is this result since by the use of another encoding where all input variables are assigned, unit resolution can easily simulate any Boolean circuit. Nevertheless, the natural encoding is the one used internally in the \textsc{sat} solvers.

The third line is related to the search for efficient \textsc{cnf} encodings of various problems in order to solve them thanks to a \textsc{sat} solver. Because unit resolution is implemented efficiently in \textsc{sat} solvers, many works aim to find encoding schemes which allow unit resolution to make as many deductions as possible. In \cite{gent-02}, a \textsc{cnf} encoding for enumerative constraints is proposed, which allows unit propagation to make the same deductions on the resulting formula as restoring arc consistency on the initial constraints does. This work was innovative because with the previously known encodings, unit propagation had less inference power than restoring arc consistency, which is the basic filtering method used in constraint solvers. It has been followed by various similar works on other kinds of constraints such as Boolean cardinality constraints \cite{bailleux-boufkhad-2003} and pseudo-Boolean constraints \cite{bailleux-boufkhad-roussel-2009}, while in \cite{Bacchus07a}, a general way to construct such an encoding for any constraint is proposed. Today, it has become customary, when a new encoding is proposed, to address the question of the behavior of unit resolution on the obtained \textsc{sat} instances. The problem is that some of these encodings produce a prohibitive number of clauses. This is why some authors seek a trade-off between the size of the formulae and the inference power of unit resolution and other deduction rules implemented in \textsc{sat} solvers, such as the failed literal rule \cite{li-chu-ambul-97}. For example, this approach is developed in \cite{sinz-05} and \cite{marquessilva-lynce-07} in the context of Boolean cardinality constraints.

\section{Propagatable functions \label{propagatablefunctions}}

\subsection{Unit resolution}

This section recalls the terminology and the principles involved in unit resolution, and introduces the notations that will be used in the rest of the paper.

A \emph{literal} is either a propositional variable or a negated propositional variable. A \emph{\textsc{cnf} formula} is defined as a conjunction $c_1\wedge\ldots\wedge c_k$ of \emph{clauses}, where each clause $c_i=l_{i,1}\vee\ldots\vee l_{i,{|c_i|}}$ is a disjunction of literals.
The \emph{size} of a \textsc{cnf} formula is its number of literal occurrences.

A \emph{truth assignment} on a set of propositional variables is a function mapping some of the variables in this set to truth values, i.e., \texttt{true} or \texttt{false}. These variables are said to be \emph{fixed} to \texttt{true} or \texttt{false}. If a truth assignment does not fix all the variables, it is said to be \emph{partial}; else it is said to be \emph{complete}. In this paper, a truth assignment will be represented as a set of literals. Given a propositional formula $\phi$ and a truth assignment $I$, $\phi|_I$ denotes $\phi \bigwedge_{l\in I}(l)$.

Any \textsc{cnf} formula $\phi$ is said to be \emph{satisfied} (falsified, respectively) by a truth assignment $I$ if and only if $I$ causes $\phi$ to evaluate to \texttt{true} (\texttt{false}, respectively) in the standard way. A \textsc{cnf} formula $\phi$ is said to be \emph{satisfiable} if and only if there exists a truth assignment that satisfy $\phi$. Any complete truth assignment satisfying a \textsc{cnf} formula $\phi$ is called a \emph{model} of $\phi$.

For convenience, a clause can be represented as a set of literals and a \textsc{cnf} formula can be represented as a set of clauses.

\begin{example}
The \textsc{cnf} formula $(a\vee \bar b)\wedge(\bar a \vee b)$ can be represented as $\{\{a, \bar b\},\{\bar a, b\}\}$.
\end{example}

\emph{Unit resolution} is an inference technique which aims either to detect an inconsistency or to assign some variables, so as to simplify a \textsc{cnf} formula. As described in a standard way by Algorithm \ref{stdUnitResolution}, until there is no empty clause and there is at least one unit clause $(w)$ in the input formula, all the occurrences of the literal $\overline w$ and all the clauses containing the literal $w$ are removed.

\begin{algorithm}[h]
\AlFnt \small
\SetKwInOut{Input}{input}
\SetKwInOut{Output}{output}
\SetKw{Return}{return}
\Input{$\phi$ [\textsc{cnf} formula];}
\Output{$(\phi, E)$ [(\textsc{cnf} formula, set of literals)] or $\bot$;}
\BlankLine
$E\leftarrow \{\}$\;
\While{$\{\}\notin \phi$ $\mathrm{and}$ $\exists \{l\}\in\phi$}
{
$\phi \leftarrow \phi \setminus \{c : c\in\phi, l\in c \} \setminus \{c :c\in\phi, \bar{l}\in c \} \cup \{c\setminus \{\bar{l}\} : c\in\phi, \bar{l}\in c \}$\;
$E \leftarrow E \cup \{l\}$\;
}
\eIf{$\{\}\in\phi$}
{
\Return $\bot$
}
{
\Return $(\phi,E)$
}
\caption{The standard algorithm for unit resolution}\label{stdUnitResolution}
\end{algorithm}

In the following, we will consider another algorithm (Algorithm \ref{altUnitResolution}) that will be called the \emph{alternative algorithm for unit resolution}. This alternative algorithm return $\bot$ if and only if the standard algorithm return $\bot$, else it returns the same truth assignment as the standard algorithm does. But contrary to the standard algorithm, it does not modify the input formula.

The alternative algorithm repeat $n+1$ \emph{propagation stages}, where $n$ is the number of variables in the input formula. Each of these propagation stage is performed by the procedure \texttt{propagation} (Algorithm \ref{propagationStage}).

\begin{algorithm}[h]
\AlFnt \small
\SetKwInOut{Input}{input}
\SetKwInOut{Output}{output}
\Input{$\phi, E$ [(\textsc{cnf} formula, set of literals)];}
\Output{$F$ [set of literals];}
\BlankLine
$F \leftarrow \{\}$\;
\ForEach{literal $w$ in $\phi$ such that $w \notin E$}
{
\ForEach{clause $(l_1\vee\cdots\vee l_k\vee w)$ of $\phi$ such that $\bar l_1,\ldots,\bar l_k\in E$}
{
$F \leftarrow F \cup \{w\}$
}
}
\Return $F$;\
\caption{The procedure \texttt{propagation}, which performs a propagation stage of the alternative algorithm for unit resolution.} \label{propagationStage}
\end{algorithm}

The standard and the alternative algorithms are strictly equivalent. The first one produces a literal $w$ if there is a unit clause $(w)$ in the simplified formula, that is if there is a clause $(l_1 \vee \cdots \vee l_q \vee w)$ in the input formula such that the literals $\bar l_1, \ldots, \bar l_q$ are previously produced. The second one produces the same literal in the same conditions, with the only difference that it does not modify the input formula. The standard algorithm return $\bot$ when an empty clause is produced. This occurs when there is some unit clause $(w)$ and the opposite clause $(\bar w)$ in the simplified formula. In the same situation, the alternative algorithm does not stop, but adds both the literals $w$ and $\bar w$ in the set $E$. As the standard algorithm, it will return $\bot$ at the end of its execution.

\SetKwFor{ForAll}{repeat}{}{end}

\begin{algorithm}[h]
\AlFnt \small
\SetKwInOut{Input}{input}
\SetKwInOut{Output}{output}
\SetKw{Repeatt}{repeat}
\SetKw{Return}{return}
\Input{$\phi$ [\textsc{cnf} formula];}
\Output{$E$ [set of literals] or $\bot$;}
\BlankLine
$E \leftarrow \{\}$\;
$n \leftarrow$ the number of variables in $\phi$\;
\ForAll() {$n+1$ times}{$E \leftarrow E \cup \mathtt{propagation}(\phi, E)$;}
\eIf{there exists $v$ such that $v, \bar v \in E$}
{
\Return $\bot$
}
{
\Return $E$
}

\caption{The alternative algorithm for unit resolution} \label{altUnitResolution}
\end{algorithm}

Of course, the alternative algorithm could be optimized in such a way that it stops if the last propagation stage did not modify the set $E$, or if there are a literal $w$ and its opposite $\bar w$ in $E$. One of these two events necessarily occurs during the first $n + 1$ propagation stages, because if $E$ contains $n+1$ literals, it necessarily contains two opposite literals.
This optimization has not been done because this algorithm is not intended to be implemented. It is only a way to prove some theoretical results which will be presented later.

In the following of the paper, given any \textsc{cnf} formula $\phi$ with $n$ variables and any integer $1\leq k\leq n+1$,
\begin{itemize}
\item $\mathcal{U}_k(\phi)$ will denote the set of literals produced  at the $k^{th}$ propagation stage of algorithm \ref{altUnitResolution};

\item $\mathcal{U}_{1..k}(\phi)$ will denote $\cup_{i=1}^k{\mathcal{U}_i(\phi)}$; 

\item $\mathcal{U}(\phi)$ will denote the result of unit propagation applied to $\phi$, i.e., either $\bot$ or $\mathcal{U}_{1..n+1}(\phi)$.
\end{itemize}

\begin{example}
With $\phi=(a\vee \bar b)\wedge(b)\wedge(\bar a\vee c\vee \bar d)$ as input, both algorithms \ref{stdUnitResolution} and \ref{altUnitResolution} return $\mathcal{U}(\phi) = \{b,a\}$. Regarding algorithm \ref{altUnitResolution}, $\mathcal{U}_{1}(\phi)=\{b\}$ and $\mathcal{U}_{2}(\phi)=\{a\}$. 
\end{example}

\subsection{Reified formulae}

This section introduces the notion of reified \textsc{cnf} formula, which will be subsequently used as a tool to prove several results. 

Informally speaking, the reified counterpart of any \textsc{cnf} formula $\phi$, is a \emph{satisfiable} \textsc{cnf} formula $\sigma=\mathrm{reif}(\phi)$ such that applying unit resolution on $\sigma$ simulates all the inferences produced by applying unit resolution on $\phi$.

Let $\mathrm{var}(\phi)$ denote the set of the variables occurring in $\phi$. 
\begin{definition}[reified formula]
Let $\phi$ be a \textsc{cnf} formula with $n$ variables. The reified counterpart of $\phi$ is the formula $\sigma=\mathrm{reif}(\phi)$ obtained as follows:
\begin{itemize}

\item There are $2n(n+2)$ variables in $\sigma$, namely $$\mathrm{var}(\sigma) = \cup_{v \in \mathrm{var}(\phi)}{\{ v_0^+,v_0^-, \ldots, v_{n+1}^+,v_{n+1}^- \}}$$
Given any propositional variable $v \in \mathrm{var}(\sigma)$, let $\delta(v)$ denotes $v^+$ and $\delta(\overline{v})$ denotes $v^-$.

\item $\sigma$ consists of the following clauses:
\begin{description}
\item[(1)] for any unary clause $(w)$ of $\phi$, $(\delta(w)_0) \wedge (\overline{ \delta(w)_0} \vee \delta(w)_1)$ \footnote{$(\delta(w)_1)$ would have the same effect in only one propagation stage, but the reified formula is intentionally tailored in such a way that the variables $v_i^+$ and $v_i^-$ are fixed at the $(i+1)^{\mathrm{th}}$ propagation stage.}, which will be called \emph{initialization clauses} of rank 0 and 1,

\item[(2)] for any stage $2\leq i\leq n+1$, and any variable $v$ of $\phi$, $(\overline{ v_{i-1}^+} \vee v_i^+) \wedge (\overline{ v_{i-1}^-} \vee v_i^-)$, which will be called \emph{propagation clauses} of rank $i$,

\item[(3)] for any stage $2\leq i\leq n+1$, and for any non-unary clause $q$ of $\phi$, $\bigwedge_{w \in q}{\chi(q,i,w)}$, where
$$\chi(q,i,w)=\delta(w)_i \vee \bigvee_{t \in q, t\neq w}{\overline{\delta( \overline{t} )_{i-1}}},$$which will be called \emph{deduction clauses} of rank $i$.

\end{description}

\end{itemize}
\end{definition}

For convenience, in the following of the paper, each time that a formula $\phi$ and it reified counterpart $\sigma = \mathrm{reif} (\phi)$ will be considered, the propagation stages on $\phi$ will be numbered from 1, while the propagation stages on $\sigma$ will be numbered from 0.

\begin{example}
Let $\phi= (a) \wedge(\overline{a}\vee b)$. The reified counterpart of $\phi$ can be decomposed as $\sigma = \sigma_{(1)} \wedge \sigma_{(2)} \wedge \sigma_{ (3)}$, where:
\begin{description}
\item[ ] $\sigma_{(1)} = (a_0^+) \wedge (\overline{a_0^+} \vee a_1^+)$

\item[ ] $\sigma_{(2)} = (\overline{a_1^+} \vee a_2^+) \wedge (\overline{a_1^-} \vee a_2^-) \wedge (\overline{b_1^+} \vee b_2^+) \wedge (\overline{b_1^-} \vee b_2^-) \wedge (\overline{a_2^+} \vee a_3^+) \wedge (\overline{a_2^-} \vee a_3^-) \wedge (\overline{b_2^+} \vee b_3^+) \wedge (\overline{b_2^-} \vee b_3^-)$

\item[ ] $\sigma_{(3)} = (\overline{a_1^+} \vee b_2^+) \wedge (\overline{b_1^-} \vee a_2^-) \wedge (\overline{a_2^+} \vee b_3^+) \wedge (\overline{b_2^-} \vee a_3^-)$
\end{description}

The stage 1 of unit resolution fixes $a$ to \texttt{true} in $\phi$. Accordingly, thanks to the initialization clauses,$_{}$ the stages 0 and 1 of unit resolution fix $a_0^+$ and $a_1^+$ to \texttt{true} in $\sigma$. 

At the stage 2, unit resolution fixes $b$ to \texttt{true} in $\phi$. Accordingly, at the stage 2 of unit resolution of $\sigma$, the variable $b_2^+$ is fixed to \texttt{true} thanks to the deduction clause $(\overline{a_1^+} \vee b_2^+)$, and the variable $a_2^+$ is fixed to \texttt{true} thanks to the propagation clause $(\overline{a_1^+} \vee a_2^+)$.

At the stage 3, unit resolution fixes no new variable in $\phi$. Thanks to the propagation clauses $(\overline{a_2^+} \vee a_3^+)$ and $(\overline{b_2^+} \vee b_3^+)$, the stage 3 of unit resolution on $\sigma$ fixes $a_3^+$ and $b_3^+$ to \texttt{true}.
\end{example}

This example shows the roles of the three kind of clauses. The initialization clauses simulate the first stage of unit resolution, which consists to fix the variables occurring in unit clauses. The propagation clauses allow unit resolution to propagate the values that where previously assigned. For example, if the variable $a$ is fixed to \texttt{true} in $\phi$ at stage 1, i.e., $a_1^+$ is fixed to \texttt{true} in $\sigma$, then the clause $( \overline{ a_1^+} \vee a_2^+)$ ensures that $a_2^+$ is fixed to \texttt{true} in $\sigma$ at stage 2. The deduction clauses allow unit resolution to simulates in $\sigma$ the deductions that are made in $\phi$. For example, if at stage 1, the variable $a$ is fixed to \texttt{true} in $\phi$, and if there is a clause $(\bar a \vee b)$ in $\phi$, then the clause $( \overline{ a_1^+} \vee b_2^+)$ of $\sigma$ allows unit resolution to fix $b_2^+$ to \texttt{true}, which indicates that unit resolution fixes $b$ to \texttt{true} in $\phi$.

Note that the proposed model of reified formula is not optimal in the sense that it usually involves redundant clauses and useless clauses. Our purpose is to provide a tool for proving theoretical results which will be presented in the following of the paper. In this context, the relevant property of the reified counterpart of a formula $\phi$ is that its size is polynomially related to the size of $\phi$. Namely, if there are $n$ variables and $k$ clauses of size at most $p$ in $\phi$, then there are $O(n^2k)$ clauses of size at most $p$ in the reified counterpart $\sigma$ of $\phi$, because for any of the $O(n)$ propagation stages, $\sigma$ contains $O(n)$ propagation clauses and $O(kn)$ deduction clauses\footnote{Because each of the $n$ variables occurs at most in $k$ clauses.}, and because the number of literals in any clause of $\sigma$ cannot exceed the size of the longest clause of $\phi$. 

\begin{definition}[reified formula with injected variables]\label{def-injected}
Let $\phi$ be any \textsc{cnf} formula with $n$ variables and $V \subseteq \mathrm{var}(\phi)$ be a set of propositional variables. The formula 
$$\mathrm{reif}(\phi, V) = \mathrm{reif}(\phi) \wedge ( \bigwedge_{v \in V} {((\overline v \vee v_1^+) \wedge (v \vee v_1^-))})$$
is said to be the reified counterpart of $\phi$ with \emph{injected} variables $V$.
The clauses added to $\mathrm{reif}(\phi)$ will be called \emph{injection clauses}.
\end{definition}

\begin{lemma} \label{lemme-reif}
Let $\phi$ be any \textsc{cnf} formula  with $n$ variables. Let $\sigma = \mathrm{ reif}(\phi)$ and $i$ be any integer such that $0 \leq i \leq n+1$. For any variable $v \in \mathrm{var}(\phi)$, applying unit resolution on $\sigma$ can fix $v_i^+$ and/or $v_i^-$ only at the propagation stage $i$, and only to \texttt{true}.
\end{lemma}

\begin{proof}
This property can be proved by induction on $i$. The only variables which can be fixed at the first propagation stage, i.e. i=0, are in $\{v_0^+, v_0^-, v \in \mathrm{ var}(\phi)\}$, because no other variables are in unary clauses, and these variables can only be fixed to \texttt{true}, because they occur positively. Now, given any $1 \leq i \leq n+1$, let us suppose the the property hold until the $(i-1)^\mathrm{th}$ propagation stage (which implies that no variable in $\{v_i^+, v_i^-, 0 \leq i \leq i-1, v \in \mathrm{ var}(\phi) \}$ has been fixed negatively). The only way for unit resolution to fix variables in $\{v_i^+, v_i^-, v \in \mathrm{ var}(\phi)\}$ is through clauses involving variables in $\{v_{i-1}^+, v_{i-1}^-, v \in \mathrm{ var}(\phi)\}$, which, by induction hypothesis, can only be fixed at the propagation stage $i-1$. Because the variables in  $\{v_{i-1}^+, v_{i-1}^-, v \in \mathrm{ var}(\phi)\}$ occur positively in these clauses, they can be only fixed to \texttt{true}.
\end{proof}

\begin{theorem}\label{theorem-reif}
Let $\phi$ be any \textsc{cnf} formula with $n$ variables and $\sigma=\mathrm{reif}(\phi)$ be the reified counterpart of $\phi$. The following properties hold:
\begin{enumerate}
\item $\sigma$ is satisfiable.

\item For any variable $v \in \mathrm{var}(\phi)$, and any integer $1\leq k \leq n+1$, the two following properties hold:

\begin{enumerate}
\item $v_k^+ \in \mathcal{U}_k(\sigma)$ if and only if $v \in \mathcal {U}_{1..k} (\phi)$;
\item $v_k^- \in \mathcal{U}_k(\sigma)$ if and only if $\bar v \in \mathcal {U}_{1..k} (\phi)$.
\end{enumerate}

\end{enumerate}
\end{theorem}

\begin{proof}
The first property arises because each clause of $\sigma$ contains at least one positive literal.

The second property can be proved by induction on $k$. For sake of brevity, let us reformulate it as follows: for any variable $v \in \mathrm{var}(\phi)$, any integer $1\leq k \leq n+1$, and any literal $w \in \{v, \overline{v}\}$, $\delta_k(w) \in \mathcal{U}_k(\sigma)$ if and only if $w \in \mathcal {U}_{1..k} (\phi)$.

Clearly, the property holds for $k=1$ because there is a clause $(w)$ in $\phi$ if and only if there is a clause $(\delta_0(w))$ and a clause $(\overline{ \delta_0(w)} \vee \delta_1(w))$ in $\sigma$. Now let us suppose that the property holds until the propagation stage $k-1$, $k>1$. 

\begin{description}
\item[$\quad\Rightarrow$]~

Let us suppose that $\delta_k(w) \in \mathcal{U}_k(\sigma)$. Then, according to the lemma \ref{lemme-reif}, one of the two following conditions hold:

\begin{enumerate}
\item There is a propagation clause $(\overline{ \delta_{k-1}(w)} \vee \delta_k(w))$ in $\sigma$ and $\delta_{k-1}(w) \in \mathcal{U}_k(\sigma)$.
By induction hypothesis, $w \in \mathcal {U}_{1..k-1} (\phi)$, then $w \in \mathcal {U}_{1..k} (\phi)$.

\item There is a deduction clause $(\overline{ \delta_{k-1}(\bar l_1)} \vee \cdots \vee \overline{ \delta_{k-1}(\bar l_q)} \vee \delta_k(w))$ in $\sigma$ and $\delta_{k-1}(\bar l_1), \ldots, \delta_{k-1}(\bar l_q) \in \mathcal {U}_{k-1} (\sigma)$.
This means that there is a clause $(l_1 \vee \cdots \vee l_q \vee w)$ in $\phi$, and, by induction hypothesis, $\bar l_1,\ldots,\bar l_q \in \mathcal{U}_{1..k-1}(\phi)$.
Then  $w \in \mathcal {U}_{1..k} (\phi)$.
\end{enumerate}

\item[$\quad\Leftarrow$]~

Let us suppose that $w \in \mathcal {U}_{1..k} (\phi)$. Then, according to the principle of unit resolution, one of the two conditions hold:

\begin{enumerate}
\item There is a clause $(w)$ in $\phi$. Then, as shown in the first part of this proof, $\delta_1(w) \in \mathcal{U}_1 (\sigma)$. Thanks to the propagation clauses $(\overline{ \delta_{1}(w)} \vee \delta_2(w)), \ldots, (\overline{ \delta_{k-1}(w)} \vee \delta_k(w))$, $\delta_k(w) \in \mathcal{U}_k(\sigma)$.

\item There is a clause $(l_1\vee\cdots\vee l_q\vee w)$ in $\phi$ and an integer $1\leq i \leq k-1$ such that $\bar l_1, \ldots, \bar l_q \in \mathcal{ U }_{ 1..i}( \phi)$. By construction of $\sigma$, there is a deduction clause $(\overline{ \delta_{i} (\bar l_1)} \vee \cdots \vee \overline{ \delta_{i}(\bar l_q)} \vee \delta_{i+1}(w))$ in $\sigma$, and by induction hypothesis, $\delta_{i}(\bar l_1), \ldots, \delta_{i}(\bar l_q) \in \mathcal {U}_{i} (\sigma)$. Then $\delta_{ i+1}( w) \in \mathcal{U}_{i+1}(\sigma)$. Either because $i+1 = k$ or thanks to the propagation clauses $(\overline{ \delta_{i}(w)} \vee \delta_{i+1}(w)), \ldots, (\overline{ \delta_{k-1}(w)} \vee \delta_k(w))$, $\delta_k(w) \in \mathcal{U }_k(\sigma)$.
\end{enumerate}
\end{description}
\end{proof}

As an interesting corollary of Theorem \ref{theorem-reif}, the failed literal rule \cite{li-chu-ambul-97}, which is a speed up technique implemented in some modern \textsc{sat} solver, can be simulated by unit propagation. Given a formula $\sigma$ and a literal $l$, the failed literal rule aims to test if $l$ must be fixed. Unit resolution is applied to $\sigma \wedge (\bar l)$ ($\sigma \wedge (l)$, respectively). If an inconsistency is detected then $l$ ($\bar l$, respectively) is fixed to \texttt{true}. The same result can be obtained by applying unit resolution on  $\mathrm{reif}(\phi \wedge (l)) \wedge \bigwedge_{v \in \mathrm{ var}(\phi)} (\overline{ v_{ n+1}^+} \vee \overline{ v_{ n+1}^-} \vee \overline l)$ ($\mathrm{reif}(\phi \wedge (\overline l)) \wedge \bigwedge_{v \in \mathrm{ var}(\phi)} (\overline{ v_{ n+1}^+} \vee \overline{ v_{ n+1}^-} \vee l)$, respectively).

\begin{theorem} \label{theorem-inject}
Let $\phi$ be any \textsc{cnf} formula, $V \subseteq \mathrm{var}(\phi)$ be a set of propositional variables, and $I \in \mathcal{I}_V$ a truth assignment. $\mathcal{U}((\mathrm{reif}(\phi, V))|_I) = \mathcal{U}(\mathrm{reif}(\phi |_I))$, i.e., unit resolution have the same effect on $\mathcal{U}((\mathrm{ reif}( \phi, V))|_I)$ as it does on $\mathcal{U}(\mathrm{reif}(\phi |_I))$.
\end{theorem}

\begin{proof}
For any literal $v \in I$ ($\bar v \in I$, respectively), the two first stages of unit resolution applied to $(\mathrm{ reif}( \phi, V))|_I$ produces the literals $v_1^+$ ($v_1^-$, respectively), which are the literals produced from the clauses $(\overline{v_0^+} \vee v_1^+), (v_0^+)$ ($(\overline{v_0^-} \vee v_1^-), (v_0^-)$, respectively) of $\mathrm{ reif}(\phi |_I)$. The other stages of unit resolution behave similarly in the two formulae, because the same clauses are involved.
\end{proof}

\subsection{Computing with unit resolution \label{computing-ur}}

This section explains how unit resolution can be used to compute functions.

\subsubsection{Definitions and terminology}

\begin{definition}[Propagator]
Let $\phi$ be a \textsc{cnf} formula, $\mathrm{var}(\phi)$ be the set of propositional variables occurring in $\phi$, $V\subseteq \mathrm{var}(\phi)$ a set of propositional variables, and $s\in V$ a propositional variable. The triplet $P = \langle\phi,V,s\rangle$ is called a \emph{propagator}. The \emph{size} of $P$ is the size of the formula $\phi$.
\end{definition}

A propagator $\langle\phi,V,s\rangle$ can act as a computer in the following way: 
\begin{itemize}
\item the input data is a partial truth assignment $I$ of the variables in $V$, i.e., some variables are assigned to \texttt{true}, some are assigned to \texttt{false}, and the other are not assigned,
\item the output can take four possible values according to the result of applying unit resolution:
\begin{itemize}
\item \texttt{fail} if $\mathcal{U}(\phi|_I) = \bot$,
\item \texttt{true} if $\mathcal{U}(\phi|_I) \neq \bot$ and $s \in \mathcal{U} (\phi|_I)$,
\item \texttt{false} if $\mathcal{U}(\phi|_I) \neq \bot$ and $\bar s \in \mathcal{U} (\phi|_I)$,
\item \texttt{na} if $\mathcal{U}(\phi|_I) \neq \bot$ and $\bar s \notin \mathcal{U} (\phi|_I)$ and $s \notin \mathcal{U} (\phi|_I)$.
\end{itemize}
\end{itemize}

Formally, $\langle\phi,V,s\rangle$ computes a function $f$ with domain $\mathcal{I}_V$ and codomain $\{\texttt{fail}$, $\texttt{true}$, $\texttt{false}$, $\texttt{na}\}$, where $\mathcal{I}_V$ denotes the set of all the consistent partial assignments on $V$, i.e., $\{I\subset V\cup\{\overline{v},v\in V\}, \forall l \in I, \bar l \notin I\}$.

Conversely, given a set $V$ of propositional variables and any function $f$ with domain $D\subseteq\mathcal{I}_V$ and codomain $\{\texttt{fail}, \texttt{true}, \texttt{false}, \texttt{na}\}$, the following issues can be addressed:
\begin{enumerate}
\item Can $f$ be computed by a propagator ?
\item If yes, how many clauses are required to compute $f$ using unit resolution ?
\end{enumerate}

These questions are important because in a \textsc{sat} solver, unit resolution is used both for detecting inconsistencies (the \texttt{fail} answer) and for inferring new information (the \texttt{true} or \texttt{false} answers), with the effect of accelerating the resolution. It is then useful to use concise \textsc{cnf} encodings which allow unit resolution to achieve as many deductions as possible. 

\begin{definition}[reified propagator]
Let $P=  \langle\phi,V,s\rangle$ be a propagator. The reified counterpart of $P$ is $\mathrm{reif}(P) = \langle \psi, V, s^{\mathrm{ true}}, s^{\mathrm{  false}}, s^{\mathrm{ fail}}\rangle$, such that $s^{\mathrm{ true}}$ is the variable $s_{n+1}^+$, $s^{\mathrm{ false}}$ is the variable $s_{n+1}^-$, $s^{\mathrm{ fail}}$ is new fresh variable, and

\[\psi = \mathrm{reif}(\phi, V) \wedge (\bigwedge_{u \in \mathrm{ var}( \phi)} {(\overline{ u_{n+1}^+} \wedge \overline{ u_{n+1}^-} \wedge s^{\texttt{\small fail}})})\]
\end{definition}

By construction of the formula $\psi$, given any $I \in \mathcal{I}_V$, applying unit resolution to $\psi|_I$ never returns $\bot$ and simulates unit resolution on $\phi|_I$ in the following sense:
\begin{itemize}
\item unit resolution on $\phi|_I$ returns $\bot$ if and only if unit resolution on $\psi|_I$ produces $s^{\mathrm{ fail}}$ (i.e. fixes to $s^{\mathrm{ fail}}$ to \texttt{true});
\item unit resolution on $\phi|_I$ produces $s$ if and only if unit resolution on $\psi|_I$ produces $s^{\mathrm{ true}}$;
\item unit resolution on $\phi|_I$ produces $\bar s$ if and only if unit resolution on $\psi|_I$ produces $s^{\mathrm{ false}}$.
\end{itemize}

\begin{definition}[filtering function]
Let $V$ be a set of propositional variables. Any function $f$ with domain $D\subseteq \mathcal{I}_V$ and codomain $\{\texttt{fail}, \texttt{true}, \texttt{false}, \texttt{na}\}$ is called a \emph{filtering function}.
\end{definition}

\begin{definition}[matching function]
Let $V$ be a set of propositional variables. Any function $f$ with domain $D\subseteq \mathcal{I}_V$ and codomain $\{\texttt{yes}, \texttt{no}\}$ is called a \emph{matching function}.
\end{definition}

Any filtering function can be specified with three matching functions in the following way:

\begin{definition}[matching functions related to a filtering function]
Let $f$ be a filtering function with domain $D \subseteq \mathcal{I}_V, V\in\{v_1, \ldots, v_n\}$. The three matching functions related to $f$ are defined as follows: 
\begin{itemize}
\item $f^{\mathrm{fail}}$ with domain $D$ and codomain $\{\texttt{yes},\texttt{no}\}$, such that for any $I\in D$, $f^{\mathrm{fail}}(I)=\texttt{yes}$ if and only if $f(I)=\texttt{fail}$,
\item $f^{\mathrm{true}}$ with domain $D^\prime=\{I\in D,f(I)\neq \texttt{fail}\}$, such that for any $I\in D^\prime$, $f^{\mathrm{true}}(I)=\texttt{yes}$ if and only if $f(I)=\texttt{true}$,
\item $f^{\mathrm{false}}$ with domain $D^\prime=\{I\in D,f(I)\neq \texttt{fail}\}$, such that for any $I\in D^\prime$, $f^{\mathrm{false}}(I)=\texttt{yes}$ if and only if $f(I)=\texttt{false}$.
\end{itemize}
\end{definition}

\begin{definition}[monotone matching function]
Given the order relation $\texttt{no}\leq_M \texttt{yes}$, any matching function $f$ with domain $D$ is said to be \emph{monotone} if  for any $I,J\in D$ such that $I\subseteq J$, $f(I)\leq_M f(J)$. 
\end{definition}

Now we will formally define the filtering and the matching functions that are computable using unit resolution.

\begin{definition}[propagatable filtering function]
Any filtering function $f$ is said to be \emph{propagatable} if and only if there exists a propagator which computes $f$.
\end{definition}

Now, two ways will be considered to compute \emph{matching} functions with unit resolution. The first one consists in using a variable as output under the assumption that $\bot$ is never returned. The second one consists in considering that the output value \texttt{yes} when $\bot$ is returned.

\begin{definition}[propagatable matching function]
Any matching function $f$ with domain $D \subseteq \mathcal{I}_V$ is said to be \emph{propagatable} if there exists a propagator $\langle\phi,V,s\rangle$ such that for any $I\in D$, the two following conditions hold:
\begin{enumerate}
\item $\mathcal{U}(\phi|_I) \neq \bot$,
\item $f(I)=\texttt{yes}$ if and only if $s \in \mathcal{U} (\phi|_I)$.
\end{enumerate}
\end{definition}

\begin{definition}[$\nu$-propagatable matching function, $\nu$-propagator]
Any matching function $f$ with domain $D \subseteq \mathcal{I}_V$ is said to be \emph{$\nu$-propagatable} if there exists a \textsc{cnf} formula $\phi$  such that for any $I\in D$, $f(I)=\texttt{true}$ if and only if $\mathcal{U} (\phi|_I) = \bot$. The couple $\langle V, \phi \rangle$ is said to be a $\nu$-propagator computing $f$.
\end{definition}

To end this necessary sequence of definitions, let us address the notion of space complexity of propagatable functions.

\begin{definition}[polynomially propagatable functions]
Let $\mathcal{F}$ be a family of filtering functions or a family of matching functions. $\mathcal{F}$ is said to be \emph{polynomially propagatable} (or polynomially $\nu$-propagatable, if applicable) if and only if any function $f\in \mathcal{F}$ with domain $D\subseteq \mathcal{I}_{\{v_1,\ldots,v_n\}}$ can be computed using a \textsc{cnf} formula of size polynomially related to $n$.
\end{definition}

\subsubsection{Propagability versus $\nu$-propagability}

In this section, we will show that propagatable and $\nu$-propagatable matching functions have the same expressive power and similar space complexities.

\begin{theorem}\label{propagatable-is-0-propagatable}
Let $f$ be a matching function. $f$ is propagatable if and only if $f$ is $\nu$-propagatable.
\end{theorem}

\begin{theorem}\label{polyprop-is-poly-0-prop}
Let $f$ be a propagatable matching function. $f$ is polynomially propagatable if and only if $f$ is polynomially $\nu$-propagatable.
\end{theorem}

\begin{proof}
$ $  
\begin{enumerate}
\item propagatable $\Rightarrow$ $\nu$-propagatable.

Let $f$ be a propagatable matching function with domain $D$, and $P=\langle\phi, V, s\rangle$ be a propagator which computes $f$. Clearly, for any partial truth assignment $I \in D$, applying unit resolution to the formula $(\phi \wedge (\overline{s}))|_I$ returns $\bot$ if and only if $f(I) = \texttt{yes}$.

\item $\nu$-propagatable $\Rightarrow$ propagatable.

Let $f$ be a $\nu$-propagatable function with domain $D\subset\mathcal{I}_V$ and $\phi$ a \textsc{cnf} formula including $n$ variables,\ such that for any $I\in D$, applying unit resolution to $\phi|_I$ returns $\bot$ if and only if $f(I)=\texttt{yes}$.

Our aim is to build a new formula $\psi$ such that for any $I\in D$, applying unit resolution to $\psi|_I$ does not return $\bot$ but fixes a variable $s$ to \texttt{true} if and only if applying unit resolution to $\phi|_I$ returns $\bot$, in such a way that the propagator $\langle \psi, V, s \rangle$ computes $f$.
 
The formula $\psi$ can be obtained as follows:

\[\psi = \mathrm{reif}(\phi, V) \wedge (\bigwedge_{u \in \mathrm{ var}( \phi)} {(\overline{ u_{n+1}^+} \wedge \overline{ u_{n+1}^-} \wedge s)})\]

The variable $s$ will be fixed to \texttt{true} if and only if unit resolution on $\phi$ fixes both a variable $u_i^+$ and a variable $u_i^-$ to \texttt{true}. According to the theorems \ref{theorem-reif} and \ref{theorem-inject}, this occurs if and only if applying unit resolution to $\phi|_I$ returns $\bot$. Then, $\langle \psi, V, s\rangle$ is a propagator which computes $f$.
\end{enumerate}

Because the two transformations have polynomial space complexity, both theorems \ref{polyprop-is-poly-0-prop} and \ref{propagatable-is-0-propagatable} hold.
\end{proof}

\subsubsection{Filtering functions versus matching functions}

In this section, we will show that without loss of generality, studying the space complexity of propagatable filtering functions reduces to studying the space complexity of propagatable matching functions.

\begin{theorem}
Any filtering function $f$ is propagatable if and only if the three related matching functions $f^{\mathrm{true}}$, $f^{\mathrm{false}}$, and $f^{\mathrm{fail}}$ are propagatable.
\end{theorem}

\begin{theorem}
Any filtering function $f$ is polynomially propagatable if and only if the three related matching functions $f^{\mathrm{true}}$, $f^{\mathrm{false}}$, and $f^{\mathrm{fail}}$ are polynomially propagatable.
\end{theorem}

\begin{proof}
$ $ 
\begin{enumerate}
\item filtering $\Rightarrow$ matching

Let $f$ be a propagatable filtering function and $f^{\mathrm{true}}$, $f^{\mathrm{false}}$, and $f^{\mathrm{fail}}$ the related matching functions. Because $f$ is propagatable,  there exists a propagator $\langle\phi,V,s\rangle$ that computes $f$. Then $f^{\mathrm{true}}$, $f^{\mathrm{false}}$, and $f^{\mathrm{fail}}$ can be computed with the following propagators, respectively:
\begin{enumerate}
\item $\langle\phi,V,s\rangle$ (which computes $f^{\mathrm{true}}$);
\item $\langle\phi\wedge(s\vee t),V,t\rangle$ (which computes $f^{\mathrm{false}}$);
\item $\langle\psi,V,s^{\mathrm{fail}}\rangle$, where $\langle \psi, V, s^{\mathrm{true}}, s^{\mathrm{false}}, s^{\mathrm{fail}}\rangle$ is the reified counterpart of $\langle \phi, V, s \rangle$.
\end{enumerate}

Clearly, $f^{\mathrm{true}}$, $f^{\mathrm{false}}$, and $f^{\mathrm{fail}}$ are propagatable. Now, because the size of $\psi$ is polynomially related to the size of $\phi$, if $f$ is polynomially propagatable then $f^{\mathrm{true}}$, $f^{\mathrm{false}}$, and $f^{\mathrm{fail}}$  are polynomially propagatable too.

\item matching $\Rightarrow$ filtering

Let $f$ be a filtering function with domain $D\subset\mathcal{I}_V, V\in\{v_1, \ldots, v_n\}$ and $f^{\mathrm{true}}$, $f^{\mathrm{false}}$, and $f^{\mathrm{fail}}$ the related matching functions. Suppose that $f^{\mathrm{true}}$, $f^{\mathrm{false}}$, and $f^{\mathrm{fail}}$ are propagatable (polynomially propagatable, respectively). Now let us consider the three following propagators (with formulae of size polynomially related to $n$, respectively):
\begin{enumerate}
\item $\langle\phi_1,V,s_1\rangle$, which computes $f^{\mathrm{true}}$;
\item $\langle\phi_2,V,s_2\rangle$, which computes $f^{\mathrm{false}}$;
\item $\langle\phi_3,V,s_3\rangle$, which computes $f^{\mathrm{fail}}$.
\end{enumerate}
Let $\langle \psi_1, V, s_1^{\mathrm{true}}, s_1^{\mathrm{false}}, s_1^{\mathrm{fail}} \rangle$, $\langle \psi_2, V, s_2^{\mathrm{true}}, s_2^{\mathrm{false}}, s_2^{\mathrm{fail}} \rangle$, $\langle \psi_3, V, s_3^{\mathrm{true}}, s_3^{\mathrm{false}}, s_3^{\mathrm{fail}} \rangle$ be the reified counterparts of these propagators.
Without loss of generality, let us suppose that, except for the input variables in $V$, the formulae $\psi_1, \psi_2, \psi_3$ have no common variable.

The function $f$ can be computed (in polynomial space, respectively) using the following propagator:

$$P=\langle \psi_1 \wedge \psi_2 \wedge \psi_3\wedge (\overline{ s_1^{ \mathrm{ true}}} \vee s) \wedge (\overline{ s_2^{ \mathrm{ false}}} \vee \overline{s}) \wedge (\overline{ s_3^{ \mathrm{ fail}}}), V, s \rangle$$

Clearly, $P$ computes $f$, which is then propagatable (polynomially propagatable, respectively).
\end{enumerate}
\end{proof}

\subsubsection{Boolean representations}

Given $V=\{v_1,\ldots,v_n\}$ a set of propositional variables, $D\subseteq \mathcal{I}_{V}$ a set of partial truth assignments of $V$, $f$ any matching function with domain $D$, and $I$ any partial assignment in $D$, let us define:

\begin{itemize}

\item the \emph{Boolean representation} of $I$ as 
$I_\mathbb{B}=(x_1,\ldots,x_n,y_1,\ldots,y_n)\in\{0,1\}^n$ such as for any $1\leq i\leq n$, $x_i=1$ if and only if $v_i \in I$, and $y_i=1$ if and only if $\overline{v_i}\in I$,

\item the \emph{Boolean representation} of $D$ as $D_{\mathbb{B}}=\{I_{\mathbb{B}},I\in D\}$,

\item the \emph{Boolean representation} of $f$ as $f_{\mathbb{B}} : D_{\mathbb{B}} \mapsto \{0,1\}$, such that for any $I\in D$, $f_\mathbb{B}(I_\mathbb{B})=1$ if and only if $f(I)=\texttt{yes}$.
\end{itemize}

\begin{example}

The following table gives an example of a matching function $f$ and its Boolean counterpart $f_{\mathbb{B}}$.

\medskip

\begin{tabular}{llll}
$I$ & $I_\mathbb{B}$ & $f(I)$ & $f_\mathbb{B}(I_\mathbb{B})$\\
\hline
$\{\overline{v_1},\overline{v_2}\}$ & $(0,0,1,1)$ & $\texttt{no}$ & 0\\
$\{\overline{v_1},v_2\}$ & $(0,1,1,0)$ & $\texttt{yes}$ & 1\\
$\{\overline{v_1}\}$ & $(0,0,1,0)$ & $\texttt{no}$ & 0\\
$\{v_1,\overline{v_2}\}$ & $(1,0,0,1)$ & $\texttt{yes}$ & 1\\
$\{v_1,v_2\}$ & $(1,1,0,0)$ & $\texttt{yes}$ & 1\\
$\{v_1\}$ & $(1,0,0,0)$ & $\texttt{yes}$ & 1\\
$\{\overline{v_2}\}$ & $(0,0,0,1)$ & $\texttt{no}$ & 0\\
$\{v_2\}$ & $(0,1,0,0)$ & $\texttt{yes}$ & 1\\
$\{\}$ & $(0,0,0,0)$ & $\texttt{no}$ & 0\\
\end{tabular}

\medskip

$\langle(\bar{v_1}\vee s)\wedge(\bar{v_2}\vee s), \{v_1,v_2\},s\rangle$ is a propagator for $f$.

\end{example}

\begin{example}

The following table gives a matching function $g$ which is \emph{not} propagatable, and its Boolean counterpart $g_{\mathbb{B}}$.

\begin{tabular}{llll}
$I$ & $I_\mathbb{B}$ & $g(I)$ & $g_\mathbb{B}(I_\mathbb{B})$\\
\hline
$\{\overline{v}\}$ & $(0,1)$ & $\texttt{yes}$ & 1\\
$\{v\}$ & $(1,0)$ & $\texttt{no}$ & 0\\
$\{\}$ & $(0,0)$ & $\texttt{yes}$ & 1\\
\end{tabular}

There is no propagator for $g$ because for any formulae $\phi_1 \subseteq \phi_2$, if $\mathcal{U}(\phi_2) \neq \bot$ then any variable fixed by unit resolution on $\phi_1$ will be fixed on $\phi_2$ as well. It follows that the third line of the table is not consistent with the second one.
\end{example}

\subsection{Synthesis}

In this section, we first introduced the notion of filtering function as a general model of functions that can be computed by unit resolution.
We then showed that any filtering function reduces to three matching functions, which are functions with binary codomain ($\{\mathrm{yes},\mathrm{no}\}$ without loss of generality) that can either be computed by unit resolution in two different ways: (1) unit resolution detects an inconsistency when the output value is \texttt{yes}, (2) it fixes a predefined output variable to \texttt{true} when the output value is \texttt{yes}. The main result of this section is that without loss of generality, studying the expressive power of unit resolution can be reduced to studying the tractability and the complexity of computing \emph{matching functions} with  unit resolution. As a corollary, in the sequel of the paper, only propagatable matching functions will be considered.

\section{Expressive power of propagators \label{expressivepower}}

Using a \emph{complete} truth assignment as input values, unit resolution has the same expressive power as Boolean circuits, because elementary gates can be directly translated into clauses. In this section, we will show that if some input variables are not fixed, the expression power of unit resolution fall down to the expression power of \emph{monotone} Boolean circuits, i.e., circuits with only \texttt{or} / \texttt{and} gates.

\subsection{Boolean circuits}

A Boolean circuit is a directed acyclic graph representing a Boolean formula. It is said to be \textit{monotone} when it contains only \texttt{and} and \texttt{or} gates.

In the following, a Boolean circuit will be represented by a triplet $\langle L,G,w\rangle$, where $L$ is a set of input labels, $w$ is the output label, and $G$ is a set of gates. A \texttt{or} gate (\texttt{and} gate, respectively) is denoted $\texttt{or}(E,t)$ ($\texttt{and}(E,t)$, respectively), where $E$ is the set of input labels and $t$ is the output label of the gate. A \texttt{not} gate is denoted $\texttt{not}(q,t)$, where $q$ is the input label of the gate and $t$ is its output label.

Given a Boolean circuit $C=\langle\{e_1,\ldots,e_n\},G,w\rangle$ and any $x=(x_1,\ldots,x_n)\in\{0,1\}^n$, let $C(x)$ denote the output value of $C$ under the assumption that its input  values are $x_1, \ldots, x_n$.
Formally, $C(x)$ can be defined as $\mathrm{val}(w)$ such that for any $1\leq i\leq n, \mathrm{val}(e_i)=x_i$, for any gate $\texttt{or}(E,t)\in G$, $\mathrm{ val}(t)=\bigvee_{e\in E}\mathrm{val}(e)$, for any gate $\texttt{and}(E,t)\in G$, $\mathrm{val}(t)=\bigwedge_{e\in E}\mathrm{val}(e)$, and for any gate $\texttt{not}(q,t)\in G$, $\mathrm{val}(t)=\neg{val}(q)$.

For convenience, an additional gate $\mathtt{tie}(q,t)$ will be used to connect two nodes $q$ and $t$ in such a way that $\mathrm{val}(t) = \mathrm{val(q)}$.

Given any Boolean function $f$ with domain $D\subseteq\{0,1\}^n$ and codomain $\{0,1\}$, any Boolean circuit $C$ with $n$ inputs is said to \emph{compute} $f$ if and only if for any $x\in D$, $C(x)=f(x)$.

\subsection{Circuits computing propagatable functions}

Because the Boolean counterpart of any matching function is a Boolean function, it can be computed by a Boolean circuit. In this section, we will show that any matching function is propagatable if and only if its Boolean counterpart can be computed using a \emph{monotone} circuit. Furthermore, we will establish a relationship between space complexity of propagatable matching functions and monotone circuit complexity.

\begin{theorem}\label{th1}
For any matching function $f$, if there exists a \emph{monotone} circuit with $n$ gates, each of them with at most $k$ inputs, which computes $f_{\mathbb{B}}$, then there exists a propagator with $O(nk)$ clauses, which computes $f$.
\end{theorem}

\begin{proof}
Let us consider any matching function $f$ with domain $D\subseteq \mathcal{I}_V$, $V=\{v_1,\ldots,v_n\},$ and any monotone circuit $Q = \langle L, G, u_k \rangle$ computing $f_{\mathbb{B}}$.
 
Without loss of generality, let us suppose that 
\begin{itemize}
\item the set of input labels of $Q$ is $L = \{e_1, \ldots, e_{2n}\}$,
\item the set of the output labels of the gates of $Q$ is $\{u_1,\ldots,u_k\}$.
\end{itemize}

Let $\tau$ be a function that maps the labels of $Q$ to propositional literals such that

$
\quad\left\{
  \begin{array}{l}
    \tau(e_i)=v_i,1\leq i\leq n\\
    \tau(e_i)=\overline{v_{i-n}},n+1\leq i\leq 2n\\
    \tau(u_i)=v_{n+i}, 1\leq i\leq k-1\\
    \tau(u_k)=s\\
  \end{array}
\right.
$

For any  gate $g = \texttt{and}(\{\alpha_1, \ldots, \alpha_m\}, t) \in G$, let $\pi(g) = (\overline{\tau( \alpha_1)} \vee \cdots \vee \overline{\tau( \alpha_m)} \vee \tau(t))$.

For any  gate $g = \texttt{or}(\{\alpha_1, \ldots, \alpha_m\}, t) \in G$, let $\pi(g) = \bigwedge_{i=1}^{m}{ (\overline{ \tau( \alpha_i)} \vee \tau(t)) }$. 

Let 
\[\phi = \bigwedge_{ g \in G}{\pi(g)}.\]

Now let us show by induction on the number $k$ of gates in $Q$ that the propagator $P=\langle\phi,V,s\rangle$ computes $f$.

The property holds for $k=0$ because if the circuit $Q$ has no gate, the output label is one of the input labels $ e_i$ or $ e_{i+n}$ related to the input variable $v_i \in V$.
If the input label is $ e_i$ then the propagator $\langle \{\}, V, v_i \rangle$ computes $f$. If the input label is $ e_{i+n}$ then the propagator $\langle (v_i \vee s), V, s \rangle$, where $s$ is a new fresh variable, computes $f$.

Now, let us suppose the the property holds for any circuit with less than $k$ clauses, $k > 0$. Let $Q = \langle L, G, u \rangle$ be any $k$-gates monotone Boolean circuit which computes the Boolean counterpart $f_\mathbb{B}$ of $f$ with input variables $\{v_1,\ldots,v_n\}$. 
Let $g$ be the output gate of $Q$. Let $\alpha_1, \ldots, \alpha_m$ be the input labels of $g$. For any $1 \leq i \leq m$, let $Q_i = \langle L, G \setminus \{ g \}, \alpha_i \rangle$. By induction hypothesis, each $Q_i$ computes the Boolean counterpart $f_{\mathbb{B}i}$ of the matching function $f_i$ computed by the propagator $P_i = \langle \phi \setminus \pi(g), V, \tau( \alpha_i) \rangle$.

Let us consider two cases:

\begin{enumerate}
\item The output gate of $Q$ is $g = \texttt{and}(\{\alpha_1, \ldots, \alpha_m\}, u)$.

Because of the nature of $g$, for any $I \in \mathcal{I}_V$, $f_\mathbb{B}(I) = 1$ if and only if for any $1 \leq i \leq m$, $f_i(I_\mathbb{B}) = 1$. Because of the nature of $\pi(g)$, $f(I) = \texttt{yes}$ if and only if for any $1 \leq i \leq m$, $\tau(\alpha_i) \in \mathcal{U}((\phi \setminus \pi(g))|_I)$. Then $f(I) = \texttt{yes}$ if and only if $f_\mathbb{B} (I_\mathbb{B}) = 1$.

\item The output gate of $Q$ is $g = \texttt{or}(\{\alpha_1, \ldots, \alpha_m\}, u)$.

Because of the nature of $g$, for any $I \in \mathcal{I}_V$, $f_\mathbb{B}(I) = 1$ if and only there exists $1 \leq i \leq m$, such as $f_i(I_\mathbb{B}) = 1$. Because of the nature of $\pi(g)$, $f(I) = \texttt{yes}$ if and only if there exists $1 \leq i \leq m$, such that $\tau(\alpha_i) \in \mathcal{U}((\phi \setminus \pi(g))|_I)$. Then $f(I) = \texttt{yes}$ if and only if $f_\mathbb{B} (I_\mathbb{B}) = 1$.
\end{enumerate}
\end{proof}

\begin{example}
The circuit of the figure \ref{circuit1} can be translated into a \textsc{cnf} formula $\phi$ in the following way:
\begin{description}
\item[] - the gate $\texttt{and}(\{e_1, e_2\}, u_1)$ produces the clause $(\bar v_1 \vee \bar v_2 \vee v_3)$;
\item[] - the gate $\texttt{or}(\{u_1, e_4\}, u_2)$ produces the clauses $(\bar v_3 \vee s)$ and $(v_2 \vee s)$.
\end{description}

This circuit computes the Boolean counterpart $f_\mathbb{B}$ of the function $f$ computed by the propagator $\langle \phi, \{v_1, v_2\}, s \rangle$.
\end{example}

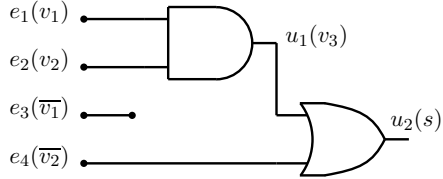
\begin{figure}[h]
\centering
\scalebox{0.8} 
{
\begin{pspicture}(0,-1.4392188)(8.502812,1.4792187)
\psdots[dotsize=0.12](2.3809376,1.1807812)
\psdots[dotsize=0.12](2.3809376,0.38078126)
\psdots[dotsize=0.12](2.3809376,-0.41921875)
\psdots[dotsize=0.12](2.3809376,-1.2192187)
\usefont{T1}{ptm}{m}{n}
\rput(1.6423438,1.2907813){$e_1 (v_1)$}
\rput{-90.0}(3.8001564,5.3617187){\psarc[linewidth=0.04](4.5809374,0.78078127){0.6}{0.0}{180.0}}
\psline[linewidth=0.04cm](4.5809374,1.3807813)(3.7809374,1.3807813)
\psline[linewidth=0.04cm](4.5809374,0.18078125)(3.7809374,0.18078125)
\psline[linewidth=0.04cm](3.7809374,1.3807813)(3.7809374,0.18078125)
\psline[linewidth=0.04cm](3.7809374,1.1807812)(3.3809376,1.1807812)
\psline[linewidth=0.04cm](3.7809374,0.38078126)(3.3809376,0.38078126)
\psline[linewidth=0.04cm](5.1809373,0.78078127)(5.5809374,0.78078127)
\usefont{T1}{ptm}{m}{n}
\rput(1.6423438,0.49078125){$e_2 (v_2)$}
\usefont{T1}{ptm}{m}{n}
\rput(1.6423438,-0.30921876){$e_3 (\overline{v_1})$}
\psline[linewidth=0.04cm](2.3809376,1.1807812)(3.3809376,1.1807812)
\psline[linewidth=0.04cm](2.3809376,0.38078126)(3.3809376,0.38078126)
\psline[linewidth=0.04cm](2.3809376,-1.2192187)(5.5809374,-1.2192187)
\psline[linewidth=0.04cm](5.5809374,0.78078127)(5.5809374,-0.41921875)
\psline[linewidth=0.04cm](2.3809376,-0.41921875)(3.1809375,-0.41921875)
\psdots[dotsize=0.12](3.1809375,-0.41921875)
\usefont{T1}{ptm}{m}{n}
\rput(7.912344,-0.50921875){$u_2 (s)$}
\usefont{T1}{ptm}{m}{n}
\rput(6.2523437,0.8907812){$u_1 (v_3)$}
\psbezier[linewidth=0.04](6.3809376,-0.21921875)(6.9649377,-0.32121876)(7.1809373,-0.41921875)(7.3809376,-0.81921875)
\psline[linewidth=0.04cm](6.3929377,-0.21921875)(5.9689374,-0.21921875)
\psbezier[linewidth=0.04](6.3809376,-1.4192188)(6.9649377,-1.3172188)(7.1809373,-1.2192187)(7.3809376,-0.81921875)
\rput{-270.0}(4.3617187,-6.0001564){\psarc[linewidth=0.04](5.1809373,-0.81921875){1.0}{-126.46923}{-53.130104}}
\psline[linewidth=0.04cm](7.3809376,-0.81921875)(7.7809377,-0.81921875)
\psline[linewidth=0.04cm](6.1009374,-0.41921875)(5.5809374,-0.41921875)
\psline[linewidth=0.04cm](6.1009374,-1.2192187)(5.5809374,-1.2192187)
\psline[linewidth=0.04cm](6.3929377,-1.4192188)(5.9689374,-1.4192188)
\usefont{T1}{ptm}{m}{n}
\rput(1.6423438,-1.1092187){$e_4 (\overline{v_2})$}
\end{pspicture} 
}
\caption{A monotone circuit computing the Boolean counterpart of a propagatable function. \label{circuit1}}
\end{figure}

\begin{theorem}\label{th2}
For any matching function $f$, if there exists a propagator $\langle \phi, V, s \rangle$ computing $f$, then there exists a \emph{monotone} circuit with $O(n^2k)$ gates computing $f_{\mathbb{B}}$, where $n$ is the number of variables and $k$ the number of clauses in $\phi$.
\end{theorem}

\begin{proof}

Let $\langle \phi, V, s \rangle$ be a propagator computing a matching function $f$. Clearly, the propagator $P = \langle \psi = \mathrm{ reif}(\phi, V), V, s_{ n+1}^+ \rangle$ computes $f$ too. According to Lemma \ref{lemme-reif} and Theorem \ref{theorem-reif},
$\psi$ can be decomposed as $\psi_0 \wedge \psi_1 \wedge \cdots \wedge \psi_{n+1}$ such that 
\begin{itemize}
\item $\psi_0$ contains the initialization clauses of rank 0 of the reified counterpart of $\phi$, 

\item $\psi_1$ contains the initialization clauses of rank 1 as well as the injection clauses,

\item for any $2 \leq i \leq i$, $\psi_i$ contains both the propagation clauses and the deduction clauses of rank $i$.
\end{itemize}

The corresponding circuit $Q$ will contain the following nodes:

\begin{itemize}
\item two \emph{input nodes} $\diamond t$ and $ \diamond \bar t$ related to each input variable $t \in V$, with the convention that $\diamond t = 1$ if and only if $t$ is assigned to \texttt{true}, and $\diamond \bar t = 1$ if and only if $t$ is assigned to \texttt{false};

\item one \emph{major node} $\diamond v$ related to any variable $v \in \mathrm{var} (\psi)$ that can be assigned to \texttt{true} by unit resolution, with the convention that $\diamond v = 1$ if and only if unit resolution fixes $v$ to \texttt{true};

\item some \emph{additional nodes}, if applicable;
\end{itemize}

Major nodes and additional nodes can be constant, i.e. permanently assigned either to 0 or 1. The constant nodes are not explicitly represented in the circuit but are referenced in the sets $U_\mathbf{0}$ and $U_\mathbf{1}$, respectively.

The circuit $Q$ consists of several layers $Q_1, \ldots, Q_{n+1}$, where each $Q_i$ simulates the stage $i$ of unit resolution on $\psi|_I$ for any $I \in \mathcal{I}_V$. 

At the first step of the construction, $U_\mathbf{0}$ is initialized with the nodes $\diamond v_0^+$ ($\diamond v_0^-$, respectively) for any variable $v \in \mathrm{ var}( \phi)$ such that $(v_0^+)$ ($(v_0^-)$, respectively) does not occur in $\psi_0$, and $U_\mathbf{1}$ is initialized with the nodes $\diamond v_0^+$ ($\diamond v_0^-$, respectively) for any variable $v_0^+$ ($v_0^-$, respectively) such that the clause $(v_0^+)$ ($(v_0^-)$, respectively) occurs in $\psi_0$.

Each of the next steps builds $Q_i$ in such a way that it simulates the effect of unit resolution applied to $\psi_i$. This is done as follows:

For each variable $v$ of $\phi$ and for each variable $u \in \{v_i^+, v_i^-\}$, let $C$ be the set of clauses of $\psi_i$ containing $u$, simplified by removing the clauses containing a literal $\bar w$ such that $\diamond w \in U_\mathbf{ 0}$ and removing any literal $\bar w$ such that $\diamond w \in U_\mathbf{1}$ from the other clauses.

If the set $C$ is empty, which means that unit resolution cannot fix $u$, then $\diamond u$ is added to $U_\mathbf{0}$. If $C$ contains a clause $(u)$, which means that unit resolution will always fix $u$ to \texttt{true}, then $\diamond u$ is added to $U_\mathbf{1}$. If $C$ contains only one clause $(\bar w \vee u)$, meaning that $u$ is fixed to \texttt{true} if and only if $w$ is previously fixed to \texttt{true}, the connection $\mathtt{ tie}( \diamond w, \diamond u)$ is produced. If $C$ contains only one clause with more than two literals like $(\overline w_1 \vee \cdots \vee \overline w_k \vee u)$, meaning that $u$ is fixed to \texttt{true} if and only if $w_1$ and ... and $w_k$ are previously fixed to \texttt{true}, the gate $\mathtt{ and}(\{\diamond w_1, \cdots, \diamond w_k\},\diamond u)$ is produced. 

In the other cases, i.e., when there are several clauses which can allow unit resolution to fix $u$,  an additional node $\alpha_c$ is created for each clause $c \in C$. For any binary clause $(\bar w \vee u) \in C$, the gate $\mathtt{ tie}( \diamond w, \alpha_c)$ is produced. For any other clause $(\overline w_1 \vee \cdots \vee \overline w_k \vee u)$, the gate  $\mathtt{ and}(\{\diamond w_1, \cdots, \diamond w_k\}, \alpha_c)$ is produced. Then the gate $\mathtt{or}(\{\alpha_c, c \in C\}, \diamond u)$ is produced, in such a way that $\mathrm{val}(\diamond u) = 1$ if and only if unit resolution fixes $u$ to \texttt{true}.

Because each sub-circuit $Q_i$ simulates exactly the effect of unit resolution on the corresponding formula $\psi_i$, the value of the output node $\diamond s_{n+1}^+$ will reflect the value of the output variable $s_{n+1}^+$ after all propagation stages on $\psi$ have been made.

The number of gates in the circuit is linearly related to the number of clauses in the reified counterpart of $\phi$, which is $O(n^2k)$. 
\end{proof}

\begin{example}\label{example-prop2circ}
Let us consider the propagator $\left\langle (a \vee \overline b \vee c), \{a, b\}, c \right\rangle$. At the first stage of the construction, $U_\mathbf{0} = \{ \diamond c_0^+, \diamond c_0^- \}$, and $U_\mathbf{1} = \{  \}$ because $\psi_0$ is empty. The input nodes of the circuit are $\diamond a, \diamond \overline a, \diamond b$, and $\diamond \overline b$.

The first layer of the circuit is based on:

\[ \psi_1 = \overbrace{(\overline a \vee a_1^+) \wedge (a \vee a_1^-) \wedge (\overline b \vee b_1^+) \wedge (b \vee b_1^-)}^{\mathrm{injection~clauses}}\]

It consists in the connections $\mathtt{tie}(\diamond a, \diamond a_1^+)$, $\mathtt{tie}(\diamond \overline a, \diamond a_1^-)$, $\mathtt{tie}(\diamond b, \diamond b_1^+)$, $\mathtt{tie}(\diamond \overline b, \diamond b_1^-)$. The variables $\diamond c_1^+$ and $\diamond c_1^-$ are added to $U_\mathbf{0}$.

The second layer is based on:

\[
\begin{array}{c}
\psi_2 = \overbrace{(\overline{a_1^+} \vee a_2^+) \wedge (\overline{a_1^-} \vee a_2^-) \wedge (\overline{b_1^+} \vee b_2^+) \wedge (\overline{b_1^-} \vee b_2^-) \wedge (\overline{c_1^+} \vee c_2^+) \wedge (\overline{c_1^-} \vee c_2^-)}^{\mathrm{propagation~clauses}} \wedge \\
\overbrace{(\overline{a_1^-} \vee \overline{b_1^+}  \vee c_2^+) \wedge (\overline{a_1^-} \vee \overline{c_1^-}  \vee b_2^-) \wedge (\overline{b_1^+} \vee \overline{c_1^-}  \vee a_2^+)}^{\mathrm{deduction~clauses}} \\
\end{array}
\]
The two last propagation clauses and the two last deduction clauses are ignored because $\diamond c_1^+$ and $\diamond c_1^-$ are in $U_\mathbf{ 0}$. The four first propagation clauses are translated into $\mathtt{tie}(\diamond a_1^+, \diamond a_2^+)$, $\mathtt{tie}(\diamond a_1^-, \diamond a_2^-)$, $\mathtt{tie}(\diamond b_1^+, \diamond b_2^+)$, $\mathtt{tie}(\diamond b_1^-, \diamond b_2^-)$.
The first deduction clause is translated into the gate $\mathtt{and}(\{a_1^-, b_1^+\},c_2^+)$. $c_2^-$ is added to $U_\mathbf{0}$.

The third layer is based on:

\[
\begin{array}{c}
\psi_3 = \overbrace{(\overline{a_2^+} \vee a_3^+) \wedge (\overline{a_2^-} \vee a_3^-) \wedge (\overline{b_2^+} \vee b_3^+) \wedge (\overline{b_2^-} \vee b_3^-) \wedge (\overline{c_2^+} \vee c_3^+) \wedge (\overline{c_2^-} \vee c_3^-)}^{\mathrm{propagation~clauses}} \wedge \\
\overbrace{(\overline{a_2^-} \vee \overline{b_2^+}  \vee c_3^+) \wedge (\overline{ a_2^-} \vee \overline{c_2^-}  \vee b_3^-) \wedge (\overline{b_2^+} \vee \overline{ c_2^-}  \vee a_3^+)}^{\mathrm{deduction~clauses}} \\
\end{array}
\]

The propagation clause $(\overline{c_2^+} \vee c_3^+)$ is translated into the connection $\mathtt{tie}(\diamond c_2^+, \alpha_1)$, the deduction clause $(\overline{a_2^-} \vee \overline{b_2^+}  \vee c_3^+)$ is translated into the gate $\mathtt{and}(\{a_2^-, b_2^+\}, \alpha_2)$, and the clause $\mathtt{ or}(\{ \alpha_1, \alpha_2 \}, \diamond c_3^+)$ is added, in such a way that $\diamond c_3^+$ is set to 1 either if $\diamond c_2^+$ is set to 1 or if both $\diamond a_2^-$ and $\diamond b_2^+$ are set to 1, i.e., if at stage 2 of unit resolution, either $c$ is fixed to \texttt{true} or $a$ and $b$ are fixed to \texttt{false} and \texttt{true}, respectively...

A part of the corresponding circuit is given Figure \ref{circuit2}. (Recall that this circuit is obtained from a reified formula, which, as mentioned above, presents some redundancies.)
\end{example}

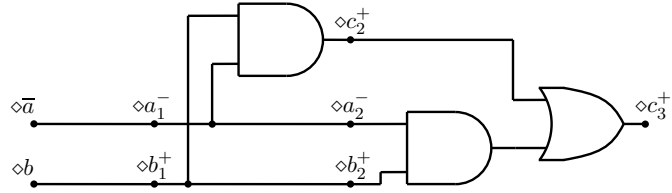
\begin{figure}[h]
\centering
\scalebox{0.8} 
{
\begin{pspicture}(0,-1.5507812)(13.482813,1.5507812)
\usefont{T1}{ptm}{m}{n}
\rput(1.6123438,-0.15921874){$\diamond \overline a$}
\psbezier[linewidth=0.04](10.599531,0.1307812)(11.183532,0.02878119)(11.399531,-0.0692188)(11.599531,-0.4692188)
\psline[linewidth=0.04cm](10.611531,0.1307812)(10.187531,0.1307812)
\psbezier[linewidth=0.04](10.599531,-1.0692189)(11.183532,-0.9672189)(11.399531,-0.86921877)(11.599531,-0.4692188)
\rput{-270.0}(8.930312,-9.868751){\psarc[linewidth=0.04](9.399531,-0.46921915){1.0}{-126.46923}{-53.130104}}
\psline[linewidth=0.04cm](11.599531,-0.4692188)(11.999532,-0.4692188)
\psline[linewidth=0.04cm](10.319531,-0.0692188)(9.76,-0.07078125)
\psline[linewidth=0.04cm](10.319531,-0.86921877)(9.799531,-0.86921877)
\psline[linewidth=0.04cm](10.611531,-1.0692189)(10.187531,-1.0692189)
\psline[linewidth=0.04cm](1.7995313,-0.46921885)(3.76,-0.47078124)
\psdots[dotsize=0.12](1.7995312,-0.4692188)
\usefont{T1}{ptm}{m}{n}
\rput(1.6123438,-1.19){$\diamond b$}
\psline[linewidth=0.04cm](1.7995313,-1.4692189)(3.76,-1.4707812)
\psdots[dotsize=0.12](1.7995312,-1.4692189)
\usefont{T1}{ptm}{m}{n}
\rput(3.7623436,-0.15921874){$\diamond a_1^-$}
\usefont{T1}{ptm}{m}{n}
\rput(3.8123438,-1.1592188){$\diamond b_1^+$}
\psdots[dotsize=0.12](3.7995312,-1.4692189)
\psdots[dotsize=0.12](3.7995312,-0.4692188)
\psline[linewidth=0.04cm](3.7995312,-0.46921885)(7.06,-0.47078124)
\psline[linewidth=0.04cm](3.7995312,-1.4692189)(7.06,-1.4707812)
\usefont{T1}{ptm}{m}{n}
\rput(7.0623436,-0.15921874){$\diamond a_2^-$}
\usefont{T1}{ptm}{m}{n}
\rput(7.112344,-1.1592188){$\diamond b_2^+$}
\rput{-90.0}(9.66875,7.9303126){\psarc[linewidth=0.04](8.799531,-0.86921877){0.6}{0.0}{180.0}}
\psline[linewidth=0.04cm](8.799531,-0.26921874)(7.9995313,-0.26921874)
\psline[linewidth=0.04cm](8.799531,-1.4692189)(7.9995313,-1.4692189)
\psline[linewidth=0.04cm](7.9995313,-0.26921874)(7.9995313,-1.4692189)
\psline[linewidth=0.04cm](7.9995313,-0.46921885)(7.06,-0.47078124)
\psline[linewidth=0.04cm](7.9995313,-1.2692188)(7.56,-1.2707813)
\psline[linewidth=0.04cm](9.399531,-0.86921877)(9.799531,-0.86921877)
\rput{-90.0}(5.06875,6.9303126){\psarc[linewidth=0.04](5.9995313,0.93078125){0.6}{0.0}{180.0}}
\psline[linewidth=0.04cm](5.9995313,1.5307813)(5.199531,1.5307813)
\psline[linewidth=0.04cm](5.9995313,0.3307812)(5.199531,0.3307812)
\psline[linewidth=0.04cm](5.199531,1.5307813)(5.199531,0.3307812)
\psline[linewidth=0.04cm](5.199531,1.3307811)(4.36,1.3292187)
\psline[linewidth=0.04cm](5.199531,0.5307812)(4.76,0.52921873)
\psline[linewidth=0.04cm](6.599531,0.93078125)(7.06,0.92921877)
\psline[linewidth=0.04cm](4.76,0.52921873)(4.76,-0.47078124)
\psline[linewidth=0.04cm](4.36,1.3292187)(4.36,-1.4707812)
\psdots[dotsize=0.12](4.36,-1.4707812)
\psdots[dotsize=0.12](4.76,-0.47078124)
\psdots[dotsize=0.12](7.06,-0.47078124)
\psdots[dotsize=0.12](7.06,-1.4707812)
\psdots[dotsize=0.12](11.96,-0.47078124)
\psdots[dotsize=0.12](7.06,0.92921877)
\usefont{T1}{ptm}{m}{n}
\rput(7.1023436,1.2407813){$\diamond c_2^+$}
\psline[linewidth=0.04cm](7.06,-1.4707812)(7.56,-1.4707812)
\psline[linewidth=0.04cm](7.56,-1.4707812)(7.56,-1.2707813)
\psline[linewidth=0.04cm](7.06,0.92921877)(9.76,0.92921877)
\psline[linewidth=0.04cm](9.76,0.92921877)(9.76,-0.07078125)
\usefont{T1}{ptm}{m}{n}
\rput(12.102344,-0.15921874){$\diamond c_3^+$}
\end{pspicture} 
}
\caption{A part of the circuit related to the example \ref{example-prop2circ}. \label{circuit2}}
\end{figure}

\begin{theorem}
Let $f$ be any matching function with domain $D$. $f$ is propagatable if and only if it is monotone.
\end{theorem}

\begin{proof}

Recall that any Boolean function $h$ with domain $D_h$ is said to be monotone if for any $z,t \in D_h$, if $z \leq_B t$ then $h(z) \leq_B h(t)$, where the ordering relation $\leq_B$ is defined as follows: $0 \leq_B 1, 0 \leq_B 0, 1 \leq_B 1$, $(z_1, \ldots z_n) \leq_B (t_1, \ldots t_n)$ if and only if $z_i \leq_B t_i, 1 \leq i \leq n$.

Now, given any matching function $f$ with domain $D$, because for any $I, J \in D$, $I \subseteq J$ if and only if $I_\mathbb{B} \leq_B J_\mathbb{ B}$, $f_\mathbb{B}$ is monotone on $D_\mathbb{B}$ if and only if $f$ is monotone on $D$.

Let $f$ be any monotone matching function with domain $D$. Because $f_\mathbb{B}$ is monotone, it can be computed by a monotone circuit. It follows from Theorem \ref{th1} that $f$ is propagatable.

Now let us consider any propagatable matching function $f$ with domain $D$. It follows from Theorem \ref{th2} that $f_\mathbb{B}$ can be computed by a monotone circuit, which implies that $f_\mathbb{B}$ is monotone. Then $f$ is monotone.
\end{proof}

\begin{theorem}\label{th_compl}
Any family $\mathcal{F}$ of propagatable functions is propagatable in polynomial space if and only if the family of the Boolean counterparts of  $\mathcal{F}$ has polynomial space monotone circuit complexity, i.e., these functions can be calculated by monotone circuits with a polynomial number of gates.
\end{theorem}

\begin{proof}
The proof of theorem \ref{th1} shows how to create a propagator from a monotone circuit. Each \texttt{and} gate with $n$ inputs is translated into one $n$-ary clause, and each \texttt{or} gate with $n$ inputs is translated into $n$ binary clauses.

The proof of theorem \ref{th2} shows how to create a monotone circuit from a propagator $P =\left\langle \phi, V, s \right\rangle $. This circuit is based on the reified counterpart $\psi$ of the formula $\phi$. Each clause of $\psi$ with $n$ literals is involved in at most $n$ \texttt{and} gates, and each literal is involved in at most one \texttt{or} gate.

\end{proof}

\section{Synthesis and perspectives \label{synthesis}}

Altogether, the results given in this paper provide important information about the expressive power of unit resolution. In particular, we can show that there exist polynomial time complexity propagatable functions that admit only propagators with an exponential number of clauses.

As an example, let us consider the Boolean functions $\mathrm{pm}^{(n)}$, like \emph{perfect matching}, such that for any $n$-bits Boolean encoding $g$ of a graph $G$, $\mathrm{pm}^{(n)}(g)=1$ if and only if there exists a perfect matching for $G$, that is a set of edges that covers each vertex exactly once. 

Next, let us consider the variants $\mathrm{vpm}^{(n)}$ such that 
\begin{itemize}
  \item the domain $D_{vpm^{(n)}}$ of $\mathrm{vpm}^{(n)}$ is the set $\{(x_1, \ldots, x_n, 0, \ldots, 0), (x_1, \ldots, x_n) \in D_{pm^{(n)}}\}$, where $D_{pm^{(n)}}$ is the domain of  $\mathrm{pm}^{(n)}$,
    \item for any $b=(x_1,\ldots,x_n,0,\ldots,0) \in D_{vpm^{(n)}}$, $\mathrm{vpm}^{(n)}(b)=1$ if and only if $\mathrm{pm}^{(n)}(x_1,\ldots,x_n)=1$.
\end{itemize}

Now, let $\mathrm{fpm}^{(n)}$ denote the matching functions related to $\mathrm{vpm}^{(n)}$. 
It is known that $\mathrm{pm}^{(n)}$, then $\mathrm{vpm}^{(n)}$, have polynomial time computational complexity but exponential monotone circuit complexity \cite{monotone-complexity-matching}. It follows from theorem \ref{th2} that $\mathrm{fpm}^{(n)}$ are filtering functions requiring an exponential number of clauses to be computed using unit resolution.

\emph{This means that although unit resolution has the same expression power as Boolean circuits regarding Boolean functions, it has a lower expression power, namely the expression power of monotone circuits, regarding \emph{filtering} functions.}

This is both very interesting and annoying, because in \textsc{sat} solvers unit propagation operates on filtering functions rather than Boolean functions. Maybe this potential weakness of unit resolution can be compensated for by other speed-up technologies. As a research perspective, this has to be verified. Meanwhile, in the field of encoding constraints into \textsc{cnf}, it would be very relevant to determine which  problems can be solved as efficiently using a \emph{simple} \textsc{sat} solver, under \textsc{cnf} encoding, as using a dedicated constraint solver maintaining generalized arc consistency. This supposes knowing the complexity of the related filtering functions regarding unit resolution.

At least two research directions follow from the ideas presented in this paper. The first one is  to characterize the expression power of some other speed-up techniques used in modern \textsc{sat} solvers, like clause learning. The second one consists in the research of deduction techniques that can polynomially compute any \textsc{cnf} encoded polynomial time complexity filtering functions.


\bibliography{biblio-propagation}
\bibliographystyle{plain}

\end{document}